\documentclass[10pt,twocolumn,letterpaper]{article}

\usepackage{cvpr}
\usepackage{times}
\usepackage{epsfig}
\usepackage{graphicx}
\usepackage{amsmath}
\usepackage{amssymb}
\usepackage{amsthm}
\usepackage{subcaption}
\usepackage{nicefrac}      
\newtheorem{theorem}{Theorem}
\usepackage[noend]{algpseudocode}
\usepackage[ruled,vlined]{algorithm2e}
\usepackage{booktabs}      
\usepackage{multirow}
\usepackage[export]{adjustbox}

\newcommand{\real}{\mathbb{R}}

\newcommand{\sj}[2]{{#2}}

\usepackage[pagebackref=true,breaklinks=true,letterpaper=true,colorlinks,bookmarks=false]{hyperref}

\cvprfinalcopy 

\def\cvprPaperID{5596} 

\ifcvprfinal\fi
\begin{document}

\title{Generative Adversarial Training by Blurring the Data Distribution Support}
\title{Generative Adversarial Training via Multi-Smoothness Support Matching}
\title{A Principled Generative Adversarial Training via Multiple Data Distribution Filterings}
\title{A Principled Generative Adversarial Training via Data Distribution Filterings}
\title{Stable Generative Adversarial Training via Data Distribution Filtering}
\title{A Stable Generative Adversarial Training via Data Distribution Filtering}
\title{On Stabilizing Generative Adversarial Training with Noise}



\author{Simon Jenni \qquad Paolo Favaro\\
University of Bern\\
{\tt\small \{simon.jenni,paolo.favaro\}@inf.unibe.ch}}

\maketitle

\begin{abstract}
We present a novel method and analysis to train generative adversarial networks (GAN) in a stable manner. As shown in recent analysis, training is often undermined by the probability distribution of the data being zero on neighborhoods of the data space. We notice that the distributions of real and generated data should match even when they undergo the same filtering. Therefore, to address the limited support problem we propose to train GANs by using different filtered versions of the real and generated data distributions. In this way, filtering does not prevent the exact matching of the data distribution, while helping training by extending the support of both distributions.
As filtering we consider adding samples from an arbitrary distribution to the data, which corresponds to a convolution of the data distribution with the arbitrary one. 
We also propose to learn the generation of these samples so as to challenge the discriminator in the adversarial training. We show that our approach results in a stable and well-behaved training of even the original minimax GAN formulation. Moreover, our technique can be incorporated in most modern GAN formulations and leads to a consistent improvement on several common datasets.
\end{abstract}

\section{Introduction}


Since the seminal work of \cite{goodfellow2014generative}, generative adversarial networks (GAN) have been widely used and analyzed due to the quality of the samples that they produce, in particular when applied to the space of natural images. Unfortunately, GANs still prove difficult to train. In fact, a vanilla implementation does not converge to a high-quality sample generator and heuristics used to improve the generator often exhibit an unstable behavior. This has led to a substantial work to better understand GANs (see, for instance, \cite{sonderby2016amortised,roth2017stabilizing,arjovsky2017towards}). In particular, \cite{arjovsky2017towards} points out how the unstable training of GANs is due to the (limited and low-dimensional) support of the data and model distributions. 
\begin{figure}[t!]
    \centering
    \begin{subfigure}[t]{.98\linewidth}
        \centering
        \includegraphics[width=\linewidth]{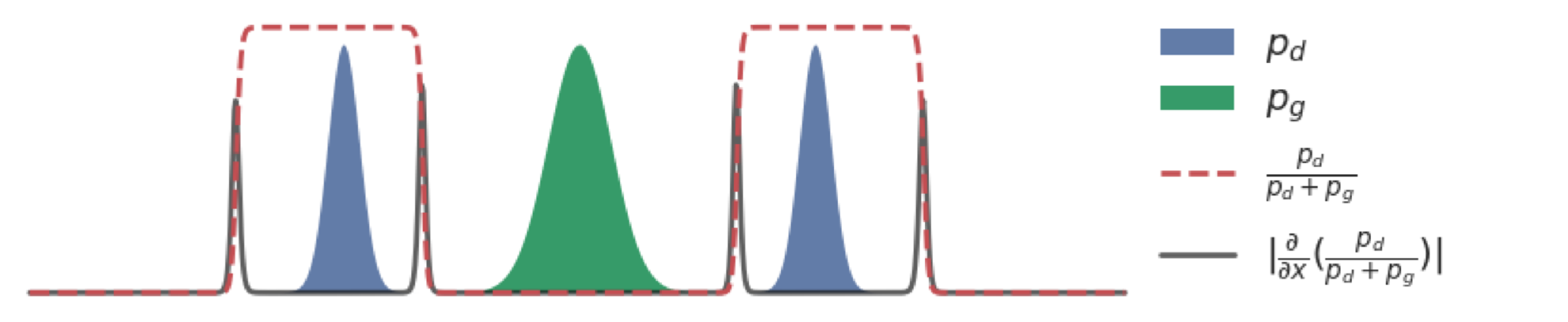}
        ~\vspace{-.8cm}
        \caption{}
        \label{fig:problem}
    \end{subfigure}
    \begin{subfigure}[t]{.98\linewidth}
        \centering
        \includegraphics[width=\linewidth]{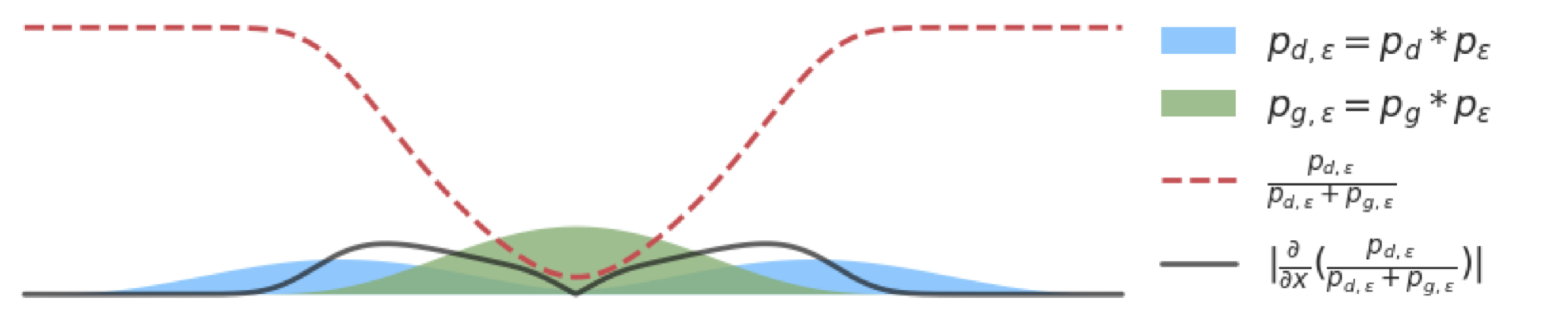}
        \vspace{-.8cm}
        \caption{}
        \label{fig:smoothing}
    \end{subfigure}
    \begin{subfigure}[t]{.98\linewidth}
        \centering
        \includegraphics[width=\linewidth]{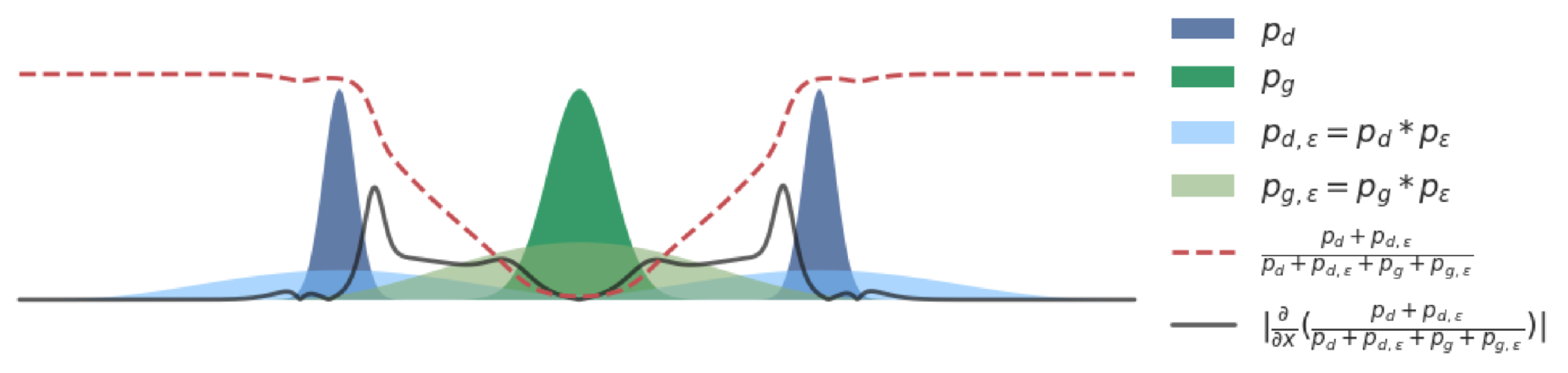}
        \vspace{-.8cm}
        \caption{}
        \label{fig:solution}
    \end{subfigure}
    \caption{
    (a) When the probability density functions of the real $p_d$ and generated data $p_g$ do not overlap, then the discriminator can easily distinguish samples. The gradient of the discriminator with respect to its input is zero in these regions and this prevents any further improvement of the generator. (b) Adding samples from an arbitrary $p_\epsilon$ to those of the real and the generated data results in the filtered versions $p_d\ast p_\epsilon$ and $p_g\ast p_\epsilon$. Because the supports of the filtered distributions overlap, the gradient of the discriminator is not zero and the generator can improve. However, the high-frequency content of the original distributions is missing. (c) By varying $p_\epsilon$, the generator can learn to match the data distribution accurately thanks to the extended supports. 
    }
  	\label{fig:sketch}
\end{figure}
In the original GAN formulation, the generator is trained against a discriminator in a minimax optimization problem. The discriminator learns to distinguish real from fake samples, while the generator learns to generate fake samples that can fool the discriminator. When the support of the data and model distributions is disjoint, the generator stops improving as soon as the discriminator achieves perfect classification, because this prevents the propagation of useful information to the generator through gradient descent (see Fig.~\ref{fig:problem}).

The recent work by \cite{arjovsky2017towards} proposes to extend the support of the distributions by adding noise to both generated and real images before they are fed as input to the discriminator. This procedure results in a smoothing of both data and model probability distributions, which indeed increases their support extent (see Fig.~\ref{fig:smoothing}). For simplicity, let us assume that the probability density function of the data is well defined and let us denote it with $p_{d}$. Then, samples $\tilde x = x+ \epsilon$, obtained by adding noise $\epsilon\sim p_\epsilon$ to the data samples $x \sim p_d$, are also instances of the probability density function $p_{d,\epsilon} = p_\epsilon \ast p_d$, where $\ast$ denotes the convolution operator. The support of $p_{d,\epsilon}$ is the Minkowski sum of the supports of $p_\epsilon$ and $p_d$ and thus larger than the support of $p_d$.
Similarly, adding noise to the samples from the generator probability density $p_g$ leads to the smoothed probability density $p_{g,\epsilon}=p_\epsilon \ast p_g$. 
Adding noise is a quite well-known technique that has been used in maximum likelihood methods, 
but is considered undesirable as it yields approximate generative models that produce low-quality blurry samples. 
Indeed, most formulations with additive noise boil down to finding the model distribution $p_{g}$ that best solves $p_{d,\epsilon}=p_{g,\epsilon}$. 
However, this usually results in a low quality estimate $p_g$ because $p_d\ast p_\epsilon$ has lost the high frequency content of $p_d$.
An immediate solution is to use a form of noise annealing, where the noise variance is initially high and is then reduced gradually during the iterations so that the original distributions, rather than the smooth ones, are eventually matched. This results in an improved training, but as the noise variance approaches zero, the optimization problem converges to the original formulation and the algorithm may be subject to the usual unstable behavior.

In this work, we 
design a novel adversarial training procedure that is stable and yields accurate results. 
We show that under some general assumptions it is possible to modify both the data and generated probability densities with additional noise without affecting the optimality conditions of the original noise-free formulation.
As an alternative to the original formulation, with $z\sim {\cal N}(0,I_d)$ and $x\sim p_{d}$,
\begin{align}
\displaystyle
\min_{G} \max_D \mathbb{E}_{x}[\log D(x)]+ \mathbb{E}_{z}[\log(1 - D(G(z)))], \label{eq:oldformulation}
\end{align}
where $D$ denotes the discriminator, we propose to train a generative model $G$ by solving instead the following optimization 
\begin{align}
\min_{G} \max_D &\sum_{p_\epsilon\in{\cal S}} \mathbb{E}_{\epsilon\sim p_\epsilon}\left[\mathbb{E}_{x\sim p_{d}}[\log D(x+\epsilon)]\right]+ \label{eq:newformulation}\\
&~~~~ \mathbb{E}_{\epsilon\sim p_\epsilon}\left[\mathbb{E}_{z\sim {\cal N}(0,I_d)}[\log(1 - D(G(z)+\epsilon))]\right],\nonumber
\end{align}
where we introduced a set $\cal S$ of probability density functions. 
If we solve the innermost optimization problem in Problem~\eqref{eq:newformulation}, then we obtain the optimal discriminator 
\begin{align}
D(x) = \frac{\sum_{p_\epsilon\in{\cal S}}p_{d,\epsilon}(x)}{\sum_{p_\epsilon\in{\cal S}}p_{d,\epsilon}(x)+p_{g,\epsilon}(x)},
\end{align}
where we have defined $p_{g}$ as the probability density of $G(z)$, where $z\sim {\cal N}(0,I_d)$. If we substitute this in the problem above and simplify we have
\begin{align}
\min_{G} \textstyle \text{JSD}\left(\frac{1}{|{\cal S}|}\sum_{p_\epsilon\in{\cal S}}p_{d,\epsilon},\frac{1}{|{\cal S}|}\sum_{p_\epsilon\in{\cal S}}p_{g,\epsilon}\right), 
\label{eq:simplifiedform}
\end{align}
where $\text{JSD}$ is the Jensen-Shannon divergence. We show that, under suitable assumptions, the optimal solution of Problem~\eqref{eq:simplifiedform} is unique and $p_g = p_d$. Moreover, since $\nicefrac{1}{|{\cal S}|}\sum_{p_\epsilon\in{\cal S}}p_{d,\epsilon}$ enjoys a larger support than $p_d$, the optimization via iterative methods based on gradient descent is more likely to achieve the global minimum, regardless of the support of $p_d$.
Thus, our formulation enjoys the following properties:
1) It defines a fitting of probability densities that is not affected by their support; 
2) It guarantees the exact matching of the data probability density function;
3) It can be easily applied to other GAN formulations.
A simplified scheme of the proposed approach is shown in Fig.~\ref{fig:dfgan}.
\begin{figure}[t!]
    \centering
        \includegraphics[width=\linewidth]{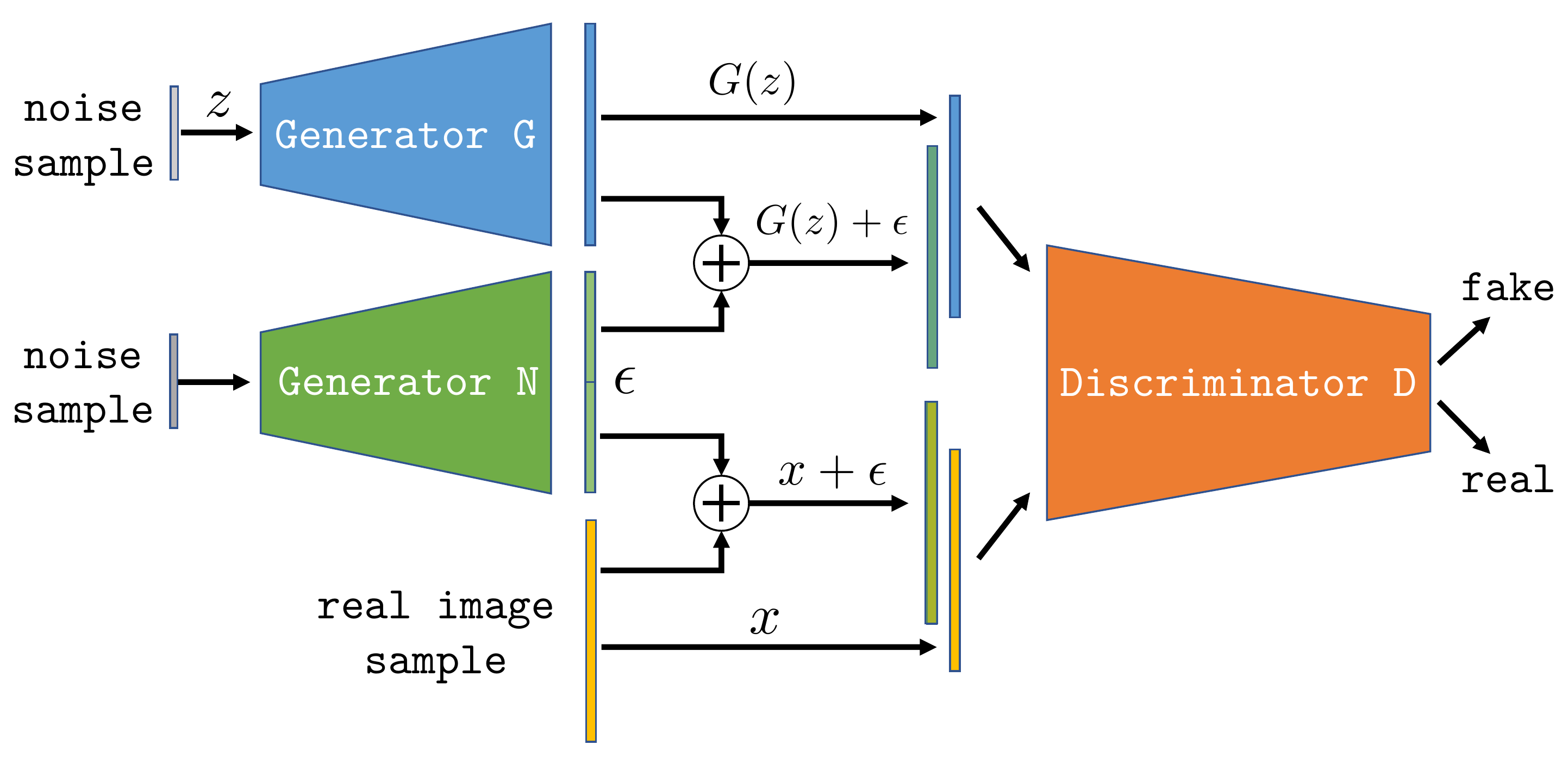}
    \caption{Simplified scheme of the proposed GAN training. We also show a noise generator $N$ that is explained in detail in Section~\ref{sec:model}. The discriminator $D$ needs to distinguish both noise-free and noisy real samples from fake ones.}
  	\label{fig:dfgan}
\end{figure}


In the next sections we introduce our analysis in detail and then devise a computationally feasible approximation of the problem formulation~\eqref{eq:newformulation}. Our method is evaluated quantitatively on CIFAR-10 \cite{krizhevsky2009learning}, STL-10 \cite{coates2011analysis}, and CelebA \cite{liu2015faceattributes}, and qualitatively on ImageNet \cite{russakovsky2015imagenet} and LSUN bedrooms \cite{Yu2015LSUNCO}.

\section{Related Work}
\label{sec:rel}

The inherent instability of GAN training was first addressed through a set of techniques and heuristics \cite{salimans2016improved} and careful architectural design choices and hyper-parameter tuning \cite{radford2015unsupervised}. \cite{salimans2016improved} proposes the use of one-sided label smoothing and the injection of Gaussian noise into the layers of the discriminator.
A theoretical analysis of the unstable training and the vanishing gradients phenomena was introduced by Arjovsky \etal \cite{arjovsky2017towards}. They argue that the main source of instability stems from the fact that the real and the generated distributions have disjoint supports or lie on low-dimensional manifolds. In the case of an optimal discriminator this will result in zero gradients that then stop the training of the generator. More importantly, they also provide a way to avoid such difficulties by introducing noise and considering ``softer'' metrics such as the Wasserstein distance. \cite{sonderby2016amortised} makes similar observations and also proposed the use of ``instance noise'' which is gradually reduced during training as a way to overcome these issues. \sj{ }{Another recent work stabilizes GAN training in a similar way by transforming examples before feeding them to the discriminator \cite{sajjadi2018tempered}. The amount of transformation is then gradually reduced during training. They  only transform the real examples, in contrast to \cite{sonderby2016amortised}, \cite{arjovsky2017towards} and our work. }  
\cite{arjovsky2017wasserstein} builds on the work of \cite{arjovsky2017towards} and introduces the Wasserstein GAN (WGAN). The WGAN optimizes an integral probability metric that is the dual to the Wasserstein distance. This formulation requires the discriminator to be Lipschitz-continuous, which is realized through weight-clipping. \cite{gulrajani2017improved} presents a better way to enforce the Lipschitz constraint via a gradient penalty over interpolations between real and generated data (WGAN-GP).
\cite{roth2017stabilizing} introduces a stabilizing regularizer based on a gradient norm penalty similar to that by \cite{gulrajani2017improved}. Its formulation however is in terms of f-divergences and is derived via an analytic approximation of adversarial training with additive Gaussian noise on the datapoints.
Another recent GAN regularization technique that bounds the Lipschitz constant of the discriminator is the spectral normalization introduced by \cite{miyato2018spectral}. This method demonstrates state-of-the-art in terms of robustness in adversarial training.
Several alternative loss functions and GAN models have been proposed over the years, claiming superior stability and sample quality over the original GAN (e.g., \cite{mao2017least}, \cite{zhao2016energy}, \cite{berthelot2017began}, \cite{arjovsky2017wasserstein}, \cite{zhao2016energy}, \cite{kodali2017train}). 
\sj{}{Adversarial noise generation has previously been used in the context of classification to improve the robustness against adversarial perturbations \cite{lee2017generative}.} 


\section{Matching Filtered Distributions}
\label{sec:convergence}

We are interested in finding a formulation that yields as optimal generator $G$ a sampler of the data probability density function (pdf) $p_d$, which we assume is well defined. The main difficulty in dealing with $p_d$ is that it may be zero on some neighborhood in the data space. An iterative optimization of Problem~\eqref{eq:oldformulation} based on gradient descent may yield a degenerate solution, \ie, such that the model pdf $p_g$ only partially overlaps with $p_d$ (a scenario called \emph{mode collapse}). It has been noticed that adding samples of an arbitrary distribution to both real and fake data samples during training helps reduce this issue. In fact, adding samples $\epsilon\sim p_\epsilon$ corresponds to blurring the original pdfs $p_d$ and $p_g$, an operation that is known to increase their support and thus their likelihood to overlap. This increased overlap means that iterative methods can exploit useful gradient directions at more locations and are then more likely to converge to the global solution.
By building on this observation, we propose to solve instead Problem~\eqref{eq:newformulation} and look for a way to increase the support of the data pdf $p_d$ without losing the optimality conditions of the original formulation of Problem~\eqref{eq:oldformulation}.

Our result below proves that this is the case for some choices of the additive noise.
We consider images of $m\times n$ pixels and with values in a compact domain $\Omega\subset\real^{m\times n}$, since image intensities are bounded from above and below. Then, also the support of the pdf $p_d$ is bounded and contained in $\Omega$. This implies that $p_d$ is also $L^2(\Omega)$.
\begin{theorem}
Let us choose ${\cal S}$ such that Problem~\eqref{eq:simplifiedform} can be written as
\begin{align}
\min_{p_g} \text{JSD}\left(\frac{1}{2}(p_{d}+p_{d}\ast p_\epsilon),\frac{1}{2}(p_g+p_g\ast p_\epsilon)\right),
\label{eq:simplifiedform2}
\end{align}
where $p_\epsilon$ is a non-degenerate probability density function in $L^2(\Omega)$.
Let us also assume that the domain of $p_g$ is restricted to $\Omega$ (and thus $p_g\in L^2(\Omega)$).
Then, the global optimum of Problem~\eqref{eq:simplifiedform2} is $p_g(x)=p_d(x)$, $\forall x\in\Omega$.
\end{theorem}
\begin{proof}
The global minimum of the Jensens-Shannon divergence is achieved if and only if
\begin{align}
p_{d}+p_{d}\ast p_\epsilon = p_g+p_g\ast p_\epsilon.
\label{eq:optimum}
\end{align}
Let $p_g = p_d + \Delta$. Then, we have $\int \Delta(x) dx = 0$ and 
$\int |\Delta(x)|^2 dx < \infty$.
By substituting $p_g$ in eq.~\eqref{eq:optimum} we obtain
$ \Delta\ast p_\epsilon = -\Delta$.
Since $\Delta$ and $p_\epsilon$ are in $L^2(\Omega)$, we can take the Fourier transform of both sides, and obtain
\begin{align}
\hat \Delta(\omega)\left(1+ \hat p_\epsilon (\omega) \right) = 0,\quad \forall\omega\in\hat\Omega.
\label{eq:optimality}
\end{align}
If $\Delta(x)\neq 0$ for some $x$, then there exists $\omega^\ast$ such that $\Delta(\omega^\ast)\neq 0$, and thus
$1+ \hat p_\epsilon (\omega^\ast) = 0$. This means that 
\begin{align}
\int p_\epsilon(x) e^{-j x^\top\omega^\ast}dx = -1
\label{eq:critical}
\end{align}
or, equivalently,
\begin{align}
\int p_\epsilon(x) \cos(x^\top \omega^\ast) dx &= -1\\
\int p_\epsilon(x) \sin(x^\top \omega^\ast) dx &= 0.
\end{align}
Notice that
\begin{align}
\int p_\epsilon(x) \cos(x^\top \omega^\ast) dx > - \int p_\epsilon(x) dx = -1
\label{eq:cond}
\end{align}
unless $p_\epsilon(x) = 0$ for any $x$ such that $x^\top \omega^\ast \neq \pi + 2k \pi$, with $k\in \mathbb{Z}$.
Since $p_\epsilon$ is not degenerate, then eq.~\eqref{eq:cond} holds, and eq.~\eqref{eq:critical} cannot be true, 
which leads to $\Delta(x) = 0$ for all $x\in \Omega$, and we can conclude that $p_g(x)=p_d(x)$, $\forall x\in\Omega$.
\end{proof}


\begin{algorithm}[!t]
\SetAlgoLined
\small
\KwIn{Training set $\mathcal{D}\sim p_d$, number of discriminator updates $n_{disc}$, number of training iterations $N$, batch-size $m$, learning rate $\alpha$, noise penalty $\lambda$} 
\KwOut{Generator parameters $\theta$}
Initialize generator parameters $\theta$, discriminator parameters $\phi$ and noise-generator parameters $\omega$ \;
\For{$1\ldots N$}{
	\For{$1\ldots n_{disc}$}{
 	   	Sample $\{x_1, \ldots ,x_m \}\sim p_d$, $\{\tilde x_1, \ldots ,\tilde x_m \}\sim p_g$ and $\{\epsilon_1, \ldots ,\epsilon_m \}\sim p_\epsilon$  \;
 	   	$L_D^r=\sum_{i=1}^m\ln(D(x_i))+\ln(D(x_i+\epsilon_i))$\;
 	   	$L_D^f=\sum_{i=1}^m\ln(1-D(\tilde x_i))+\ln(1-D(\tilde x_i+\epsilon_i))$\;
 	   	$L_\epsilon=\sum_{i=1}^m |\epsilon_i|^2 $\;
 	   	$\phi \leftarrow \phi + \nabla_\phi L_D^r(\phi, \omega) + \nabla_\phi L_D^f(\phi, \omega)$\;
 	   	$\omega \leftarrow \omega - \nabla_\omega \big(L_D^r(\phi, \omega)+L_D^f(\phi, \omega)+\lambda L_\epsilon (\omega) \big)$\;
	}
  	Sample $\{\tilde x_1, \ldots ,\tilde x_m \}\sim p_g$ and $\{\epsilon_1, \ldots ,\epsilon_m \}\sim p_\epsilon$ \;
	$L_G^f=\sum_{i=1}^m\ln(D(\tilde x_i))+\ln(D(\tilde x_i+\epsilon_i))$\;

 	$\theta \leftarrow \theta + \nabla_\theta L_G^f(\theta)$\;
}
\caption{\small Distribution Filtering GAN (DFGAN)}
\label{alg:1}
\end{algorithm}

\subsection{Formulation}
\label{sec:model}

Based on the above theorem we consider two cases: 
\begin{enumerate}
\item \textbf{Gaussian noise} with a fixed/learned standard deviation $\sigma$: 
$
p_\epsilon(\epsilon) = {\cal N}(\epsilon;0,\sigma I_d)
$;
\item\textbf{Learned noise} from a noise generator network $N$ with parameters $\sigma$: 
$
p_\epsilon(\epsilon) \text{ such that } \epsilon = N(w,\sigma), \text{ with } w\sim {\cal N}(0,I_d).
$
\end{enumerate}
In both configurations we can learn the parameter(s) $\sigma$. We do so by minimizing the cost function after the maximization with respect to the discriminator. The minimization encourages large noise since this would make $p_{d,\epsilon}(\omega)$ more similar to $p_{g,\epsilon}(\omega)$ regardless of $p_d$ and $p_g$. This would not be very useful to gradient descent. Therefore, to limit the noise magnitude we introduce as a regularization term the noise variance $\Gamma(\sigma) = \sigma^2$ or the Euclidean norm of the noise output image $\Gamma(\sigma) = \mathbb{E}_{w\sim {\cal N}(0,I_d)}|N(w,\sigma)|^2$, and multiply it by a positive scalar $\lambda$, which we tune.

The proposed formulations can then be written in a unified way as:
\begin{align}
\begin{aligned}	
\min_{G}  \min_{\sigma} \max_D  \lambda \Gamma+
\mathbb{E}_{x}\Big[\log D(x)+\mathbb{E}_{\epsilon}\log D(x+\epsilon)\Big] + \\ 
\mathbb{E}_{z}\Big[\log[1 - D(G(z))]+\mathbb{E}_{\epsilon}\log[1-D(G(z)+\epsilon)]\Big]. 
\end{aligned}
\label{eq:formulationGauss}
\end{align}

\subsection{Implementation}\label{sec:impl}
Implementing our algorithm only requires a few minor modifications of the standard GAN framework. We perform the update for the noise-generator and the discriminator in the same iteration. Mini-batches for the discriminator are formed by collecting all the fake and real samples in two separate batches, \ie, $\{x_1, \ldots ,x_m, x_1+\epsilon_1, \ldots ,x_m+\epsilon_m \}$ is the batch with real examples and $\{\tilde x_1, \ldots ,\tilde x_m, \tilde x_1+\epsilon_1, \ldots ,\tilde x_m+\epsilon_m \}$ the fake examples batch. The complete procedure is outlined in Algorithm~\ref{alg:1}.
The noise-generator architecture is typically the same as the generator, but with a reduced number of convolutional filters. Since the inputs to the discriminator are doubled when compared to the standard GAN framework, the DFGAN framework can be $1.5$ to $2$ times 
slower. Similar and more severe performance drops are present in existing variants (\eg, WGAN-GP). Note that by constructing the batches as $\{x_1, \ldots ,x_{m/2}, x_{m/2+1}+\epsilon_1, \ldots ,x_m+\epsilon_m \}$ the training time is instead comparable to the standard framework,
but it is much more stable and yields an accurate generator. For a comparison of the runtimes, see Fig.~\ref{fig:converg}.

 \begin{figure}[t]
    \centering
    \begin{subfigure}{.49\linewidth}
        \centering
        \includegraphics[width=\linewidth]{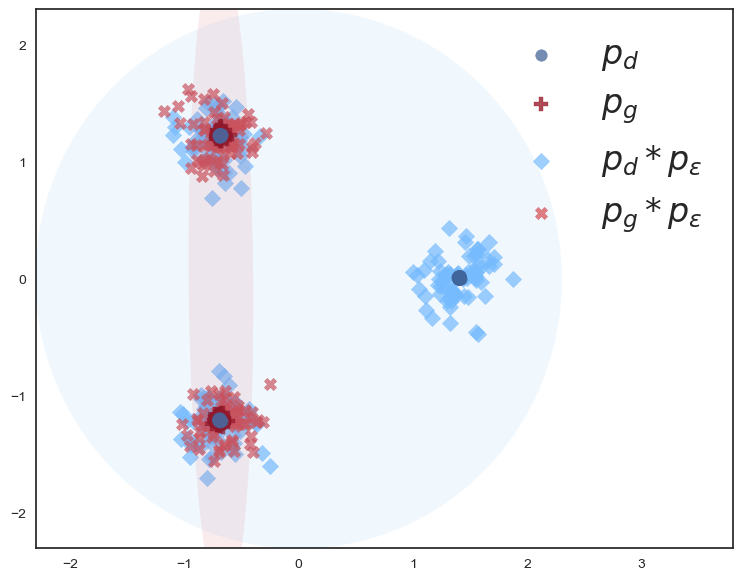}
        \caption{}\label{fig:nobn}
    \end{subfigure}
    \begin{subfigure}{.49\linewidth}
        \centering
        \includegraphics[width=\linewidth]{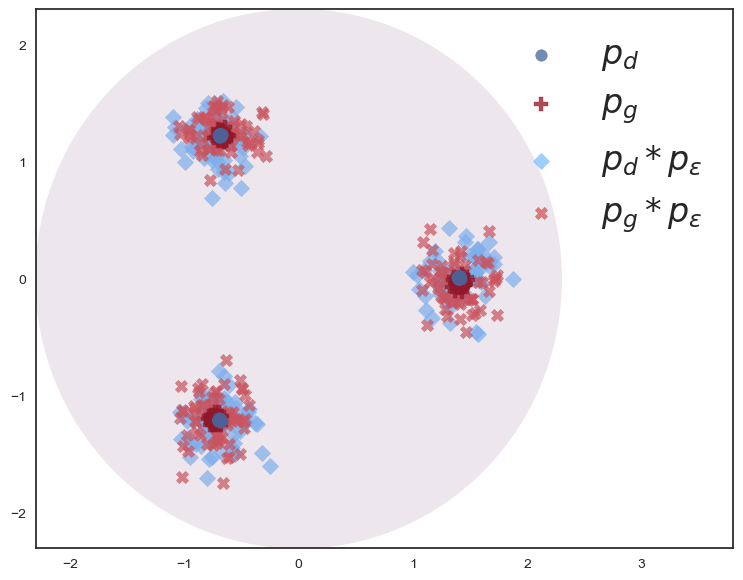}
        \caption{}\label{fig:nobn}
    \end{subfigure}
    \caption{Illustration of how separate normalization of fake and real mini-batches discourages mode collapse. In (a) no normalization is applied and mode collapse is observed. Since the covered modes are indistinguishable, the generator receives no signal that encourages better mode coverage. In (b) separate normalization of the real and fake data is applied. The mismatch in the batch statistics (mean and standard deviation) can now be detected by the discriminator, forcing the generator to improve.   }
  	\label{fig:bn}
\end{figure}

\subsection{Batch-Normalization and Mode Collapse}\label{sec:bn}
The current best practice is to apply batch normalization to the discriminator separately on the real and fake mini-batches \cite{Soumith2016}. Indeed, this showed much better results when compared to feeding mini-batches with a 50/50 mix of real and fake examples in our experiments. The reason for this is that batch normalization implicitly takes into account the distribution of examples in each mini-batch. To see this, consider the example in Fig.~\ref{fig:bn}. In the case of no separate normalization of fake and real batches we can observe mode-collapse. The modes covered by the generator are indistinguishable for the discriminator, which observes each example independently. There is no signal to the generator that leads to better mode coverage in this case. Since the first two moments of the fake and real batch distribution are clearly not matching, a separate normalization will help the discriminator distinguish between real and fake examples and therefore encourage better mode coverage by the generator.

Using batch normalization in this way turns out to be crucial for our method as well. Indeed, when no batch normalization is used in the discriminator, the generator will often tend to produce noisy examples. This is difficult to detect by the discriminator, since it judges each example independently. To mitigate this issue we apply separate normalization of the noisy real and fake examples before feeding them to the discriminator. We use this technique for models without batch normalization (\eg SNGAN).

\begin{figure}[t]
    \centering
    \includegraphics[width=1\linewidth,trim=1.1cm 0 1.5cm 1.45cm,clip ]{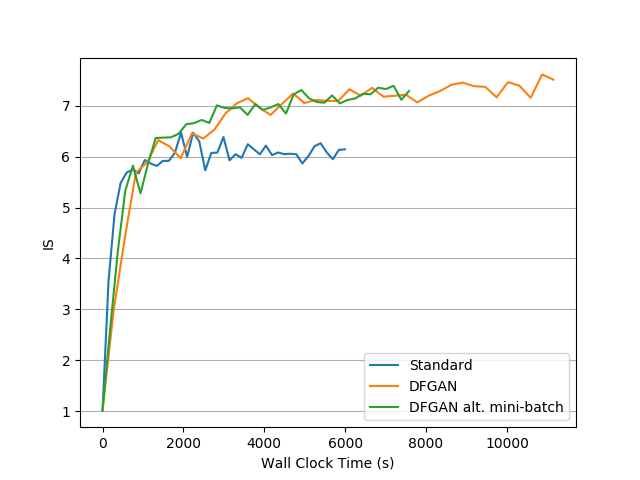}
    \caption{\sj{}{A comparison of wall clock time vs \texttt{IS} for GANs with and without distribution filtering. The models use the architecture specified in Table \ref{tab:cifar_net} and were trained on CIFAR-10. The computational overhead introduced by our method does not negatively affect the speed of convergence.} }
  	\label{fig:converg}
\end{figure}

\section{Experiments}
\label{sec:exp}

We compare and evaluate our model using two common GAN metrics: the Inception score \texttt{IS} \cite{salimans2016improved} and the Fr\'echet Inception distance \texttt{FID} \cite{heusel2017gans}. Throughout this section we use 10K generated and real samples to compute \texttt{IS} and \texttt{FID}. In order to get a measure of the stability of the training we report the mean and standard deviation of the last five checkpoints for both metrics (obtained in the last 10\% of training). 
More reconstructions, experiments and details are provided in the supplementary material.

\subsection{Ablations} 
To verify our model we perform ablation experiments on two common image datasets: CIFAR-10 \cite{krizhevsky2009learning} and STL-10 \cite{coates2011analysis}. For CIFAR-10 we train on the 50K $32\times32$ RGB training images and for STL-10 we resize the 100K $96\times96$ training images to $64\times64$. The network architectures resemble the DCGAN architectures of \cite{radford2015unsupervised} and are detailed in Table \ref{tab:cifar_net}. All the models are trained for 100K generator iterations using a mini-batch size of 64. We use the ADAM optimizer \cite{kingma2014adam} with a learning rate of $10^{-4}$ and $\beta_1=0.5$. Results on the following ablations are reported in Table \ref{tab:abl}:

\begin{table}[t]
\centering
\caption{Network architectures used for experiments on CIFAR-10 and STL-10. Images are assumed to be of size $32\times32$ for CIFAR-10 and $64\times64$ for STL-10. We set $M=512$ for CIFAR-10 and $M=1024$ for STL-10. Layers in parentheses are only included for STL-10. The noise-generator network follows the generator architecture with the number of channels reduced by a factor of 8. BN indicates the use of batch-normalization \cite{ioffe2015batch}.  }
\label{tab:cifar_net}
\footnotesize
\begin{tabular}{@{}l@{\hspace{.15cm}}c@{}}
\begin{tabular}{@{}l@{}}
\toprule 
\textbf{Generator CIFAR-10/(STL-10) }                     \\ \midrule
$z \in \mathbb{R}^{128} \sim \mathcal{N}(0, I)$ \\ 
fully-conn. BN ReLU $4\times 4\times M$    \\ 
(deconv $4\times4$ str.=2 BN ReLU 512)			\\
deconv $4\times4$ str.=2 BN ReLU 256          \\ 
deconv $4\times4$ str.=2 BN ReLU 128          \\ 
deconv $4\times4$ str.=2 BN ReLU 64           \\ 
deconv $3\times3$ str.=1 \texttt{tanh} 3                 \\ \bottomrule
\end{tabular}
&
\begin{tabular}{@{}l@{}}
\toprule
\textbf{Discriminator CIFAR-10/(STL-10)}       \\ \midrule
conv $3\times3$ str.=1 lReLU 64     \\
conv $4\times4$ str.=2 BN lReLU 64  \\
conv $4\times4$ str.=2 BN lReLU 128 \\
conv $4\times4$ str.=2 BN lReLU 256 \\
conv $4\times4$ str.=2 BN lReLU 512 \\
(conv $4\times4$ str.=2 BN lReLU 1024) \\
fully-connected \texttt{sigmoid} 1                \\ \bottomrule
\end{tabular}
\end{tabular}

\end{table}

\begin{description}
	\item [(a)-(c) Only noisy samples:] In this set of experiments we only feed noisy examples to the discriminator. In experiment \textbf{(a)} we add Gaussian noise and in \textbf{(b)} we add learned noise. In both cases the noise level is not annealed. While this leads to stable training, the resulting samples are of poor quality which is reflected by high \texttt{FID} and low \texttt{IS}. The generator will tend to also produce noisy samples since there is no incentive to remove the noise.
	 Annealing the added noise during training as proposed by \cite{arjovsky2017towards} and \cite{sonderby2016amortised} leads to an improvement over the standard GAN. This is demonstrated in experiment \textbf{(c)}. The added Gaussian noise is linearly annealed during the 100K iterations in this case;
	\item [(d)-(i) Both noisy and clean samples:] The second set of experiments consists of variants of our proposed model. Experiments \textbf{(d)} and \textbf{(e)} use a simple Gaussian noise model; in \textbf{(e)} the standard deviation of the noise $\sigma$ is learned. We observe a drastic improvement in the quality of the generated examples even with this simple modification. The other experiments show results of our full model with a separate noise-generator network. We vary the weight $\lambda$ of the $L^2$ norm of the noise in experiments \textbf{(f)-(h)}. Ablation \textbf{(i)} uses the alternative mini-batch construction with faster runtime as described in Section \ref{sec:impl};
\end{description}
%
%
%
%
\begin{table*}[h]
\centering
\caption{We perform ablation experiments on CIFAR-10 and STL-10 to demonstrate the effectiveness of our proposed algorithm. Experiments (a)-(c) show results where only filtered examples are fed to the discriminator. Experiment (c) corresponds to previously proposed noise-annealing and results in an improvement over the standard GAN training. Our approach of feeding both filtered and clean samples to the discriminator shows a clear improvement over the baseline. }
\label{tab:abl}
\begin{tabular*}{\textwidth}{@{}l@{\extracolsep{\fill}}cccc@{}}
\toprule 
\multirow{2}{*}{Experiment} & \multicolumn{2}{c}{\textbf{CIFAR-10}} & \multicolumn{2}{c}{\textbf{STL-10}} \\
                            & \texttt{FID}      & \texttt{IS} 		& \texttt{FID}        & \texttt{IS}        			\\ \midrule 
Standard GAN   &   $46.1\pm0.7$          &     $6.12\pm.09$       &    $78.4\pm6.7$     &  $8.22\pm.37$                \\ \midrule
(a) Noise only: $\epsilon \sim \mathcal{N}(0, I)$   &   $94.9\pm4.9$    &  $4.68\pm.12$  & $107.9\pm2.3$    &  $6.48\pm.19$     \\ 
(b) Noise only: $\epsilon$ learned    &   $69.0\pm3.4$    &  $5.05\pm.14$    &   $107.2\pm3.4$        &     $6.39\pm.22$            \\
(c) Noise only: $\epsilon \sim \mathcal{N}(0,\sigma I)$, $\sigma \rightarrow 0$ & $44.5\pm3.2$  &  $6.85\pm.20$ &  $75.9\pm1.9$  &  $8.49\pm.19$                      \\
\midrule
(d) Clean + noise: $\epsilon \sim \mathcal{N}(0,I)$   &   $29.7\pm0.6$  &  $7.16\pm.05$  &  $66.5\pm2.3$   & $8.64\pm.17$    \\
(e) Clean + noise: $\epsilon \sim \mathcal{N}(0,\sigma I)$ with learnt $\sigma$   &   $28.8\pm0.7$  &  $7.23\pm.14$  &  $71.3\pm1.7$   & $8.30\pm.12$    \\
(f) DFGAN $(\lambda=0.1 )$   &     $27.7\pm0.8$   &  $7.31\pm.06$    &     $\textbf{63.9}\pm1.7$       &       $\textbf{8.81}\pm.07$                 \\ 
(g) DFGAN $(\lambda=1 )$     &     $\textbf{26.5}\pm0.6$   &  $\textbf{7.49}\pm.04$    &      $64.0\pm1.4$      &         $8.52\pm.16$               \\ 
(h) DFGAN $(\lambda=10 )$    &      $29.8\pm0.4$  &  $6.55\pm.08$    &       $66.9\pm3.2$     &          $8.38\pm.20$              \\ 
\midrule
(i) DFGAN alt. mini-batch $(\lambda=1)$   &      $28.7\pm0.6$  &  $7.3\pm.05$    &       $67.8\pm3.2$     &          $8.30\pm.11$              \\ 
\bottomrule
\end{tabular*}
\end{table*}

\noindent\textbf{Application to Different GAN Models.} We investigate the possibility of applying our proposed training method to several standard GAN models. The network architectures are the same as proposed in the original works with only the necessary adjustments to the given image-resolutions of the datasets (\ie, truncation of the network architectures). The only exception is SVM-GAN, where we use the architecture in Table \ref{tab:cifar_net}. Note that for the GAN with minimax loss (MMGAN) and WGAN-GP we use the architecture of DCGAN. Hyper-parameters are kept at their default values for each model. The models are evaluated on two common GAN benchmarks: CIFAR-10 \cite{krizhevsky2009learning} and CelebA \cite{liu2015faceattributes}. The image resolution is $32\times32$ for CIFAR-10 and $64\times64$ for CelebA. All models are trained for 100K generator iterations. For the alternative objective function of LSGAN and SVM-GAN we set the loss of the noise generator to be the negative of the discriminator loss, as is the case in our standard model. The results are shown in Table \ref{tab:appl}. We can observe that applying our training method improves performance in most cases and even enables the training with the original saturation-prone minimax GAN objective, which is very unstable otherwise. Note also that applying our method to SNGAN \cite{miyato2018spectral} (the current state-of-the-art) leads to an improvement on both datasets. We also evaluated SNGAN with and without our method on $64 \times 64$ images of  STL-10 (same as in Table \ref{tab:abl}) where our method boosts the performance from an FID of $66.3\pm1.1$ to $58.3\pm1.4$. We show random CelebA reconstructions from models trained with and without our approach in Fig.~\ref{fig:celeb_comp}.

\begin{table}[t]
\centering
\caption{We apply our proposed GAN training to various previous GAN models trained on CIFAR-10 and CelebA. The same network architectures and hyperparameters as in the original works are used (for SVM-GAN we used the network in Table \ref{tab:cifar_net}). We observe that our method increases performance in most cases even with the suggested hyperparameter settings. Note that our method also allows successful training with the original minimax MMGAN loss as opposed to the commonly used heuristic (\eg, in DCGAN). }
\label{tab:appl}
\resizebox{\linewidth}{!}{\begin{tabular}{@{}lccc@{}}
\toprule 
\multirow{2}{*}{Model} & \multicolumn{2}{c}{\textbf{CIFAR-10}} & \textbf{CelebA} \\
                            & \texttt{FID}      & \texttt{IS} 		& \texttt{FID}            			\\ \midrule 
MMGAN \cite{goodfellow2014generative}  &   $>450$          &     $\sim1$       &     $>350$         \\ 
DCGAN \cite{radford2015unsupervised}   &   $33.4\pm0.5$          &     $6.73\pm.07$       &     $25.4\pm2.6$         \\ 
WGAN-GP \cite{gulrajani2017improved}   &   $37.7\pm0.4$          &     $6.55\pm.08$       &     $15.5\pm0.2$       \\ 
LSGAN \cite{mao2017least}  &   $38.7\pm1.8$          &     $6.73\pm.12$       &     $21.4\pm1.1$    \\ 
SVM-GAN \cite{lim2017geometric}  &   $43.9\pm1.0$          &     $6.25\pm.09$       &     $26.5\pm1.9$    \\ 

SNGAN (\cite{miyato2018spectral}  &   $29.1\pm0.4$          &     $7.26\pm.06$       &     $13.2\pm0.3$        \\ \midrule
MMGAN +DF ($\lambda=0.1$)  &   $33.1\pm0.7$          &     $6.91\pm.05$       &     $16.6\pm1.9$      \\ 
DCGAN + DF ($\lambda=10$)  &   $31.2\pm0.3$     &    $6.95\pm.11$     &     $14.7\pm1.0$       \\ 
LSGAN + DF  ($\lambda=10$)  &   $36.7\pm1.2$        &     $6.63\pm.17$     &     $19.9\pm0.4$        \\
SVM-GAN + DF  ($\lambda=1$)  &   $28.7\pm1.1$          &     $7.31\pm.11$       &     $12.7\pm0.7$         \\  
SNGAN + DF ($\lambda=1$)  & $\textbf{25.9}\pm0.3$      &    $\textbf{7.47}\pm.08$      &     $\textbf{10.5}\pm0.4$        \\ 
\bottomrule
\end{tabular}}
\end{table}

\begin{figure*}[t!]
    \centering
    \begin{subfigure}[t]{0.5\linewidth}
        \centering
        \adjincludegraphics[width=8.5cm,trim={0 0 0 {.0\height}},clip]{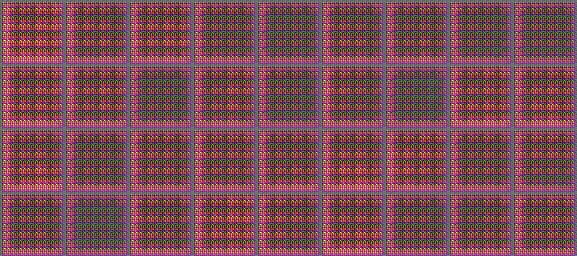}
        \caption{Original GAN without DF}
    \end{subfigure}%
    \begin{subfigure}[t]{0.5\linewidth}
        \centering
        \adjincludegraphics[width=8.5cm,trim={0 0 0 {.0\height}},clip]{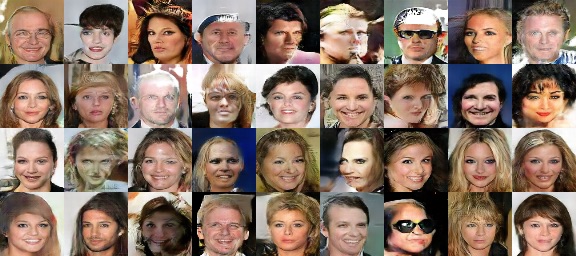}

        \caption{Original GAN with DF}
    \end{subfigure}
    \begin{subfigure}[t]{0.5\linewidth}
        \centering
        \adjincludegraphics[width=8.5cm,trim={0 0 0 {.0\height}},clip]{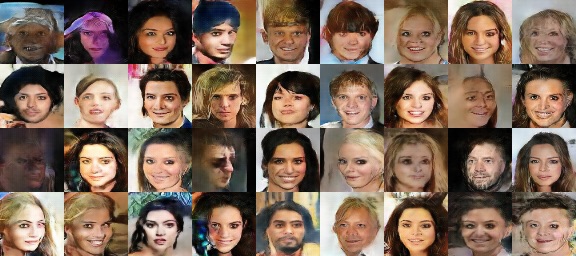}
        \caption{DCGAN without DF}
    \end{subfigure}%
    \begin{subfigure}[t]{0.5\linewidth}
        \centering
        \adjincludegraphics[width=8.5cm,trim={0 0 0 {.0\height}},clip]{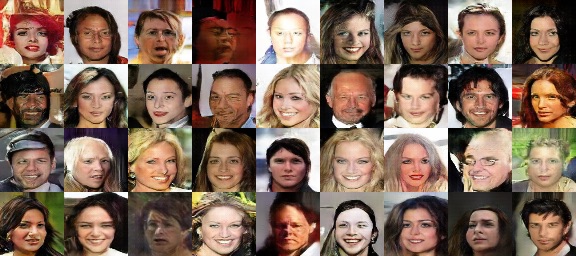}
        \caption{DCGAN with DF}
    \end{subfigure}
    \caption{Left column: Random reconstructions from models trained on CelebA without distribution filtering (DF). Right column: Random reconstructions with our proposed method.}
  	\label{fig:celeb_comp}
\end{figure*}

\begin{figure*}[t!]
    \centering
    \begin{subfigure}[t]{\textwidth}
        \centering
 		\adjincludegraphics[width=\linewidth,trim={0 {.25\height} 0 0},clip]{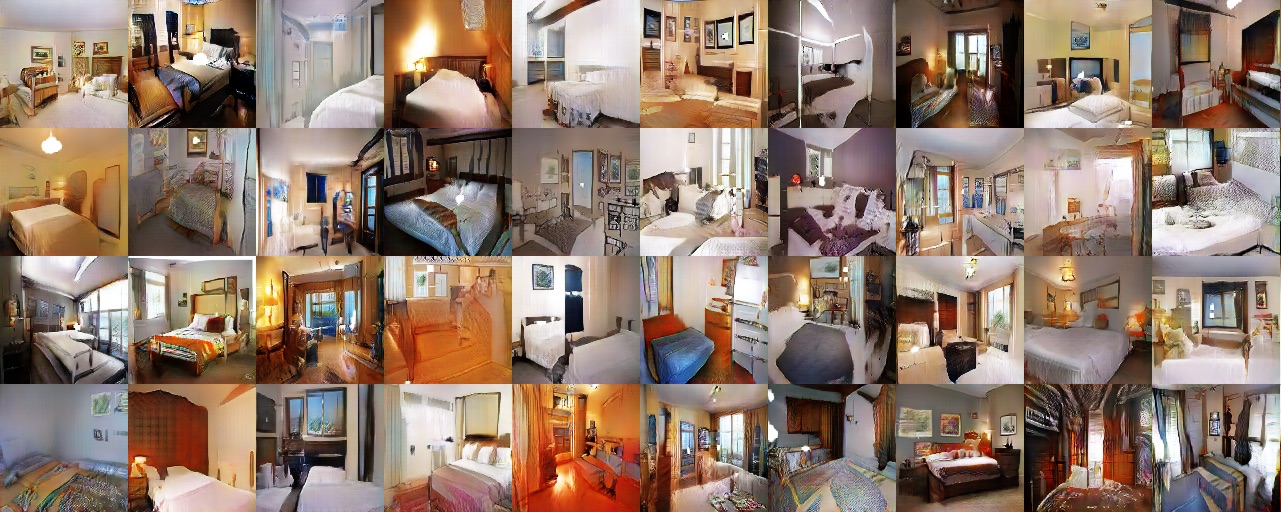}

    \end{subfigure}\vspace{0.1cm}
    \begin{subfigure}[t]{\textwidth}
        \centering
 		\adjincludegraphics[width=\linewidth,trim={0 {.25\height} 0 0},clip]{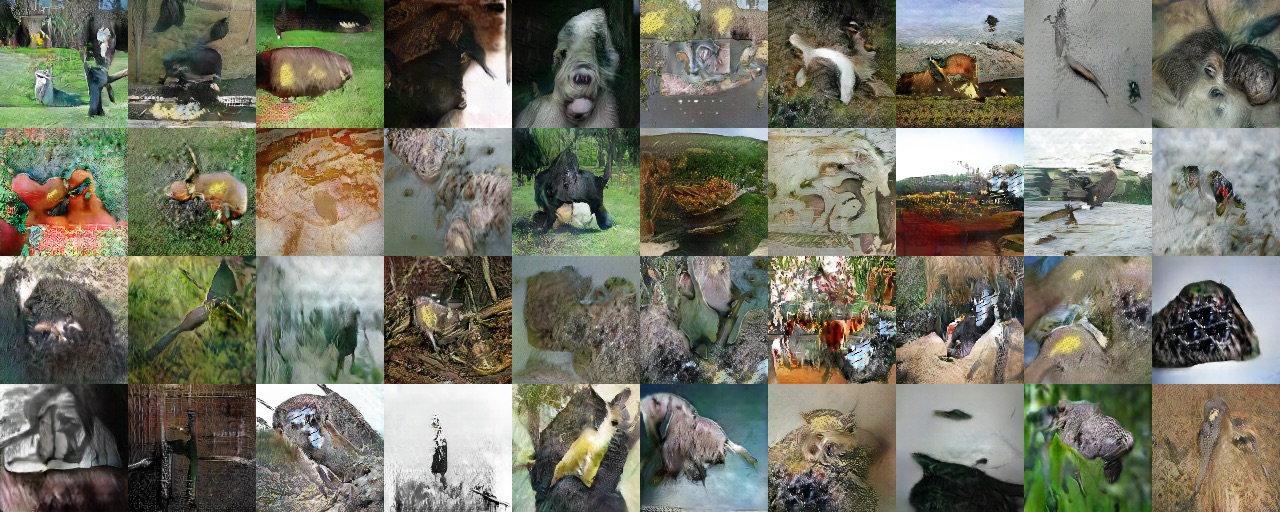}

    \end{subfigure}
    \caption{Reconstructions from DFGANs trained on $128 \times 128$ images from the LSUN bedrooms dataset (\emph{top}) and ImageNet (\emph{bottom}).}
  	\label{fig:lsun_imnet}
\end{figure*}
\noindent\textbf{Robustness to Hyperparameters.} We test the robustness of DFGANs with respect to various hyperparamters by training on CIFAR-10 with the settings listed in Table~\ref{tab:set}. The network is the same as specified in Table~\ref{tab:cifar_net}. The noise penalty term is set to $\lambda=0.1$. We compare to a model without our training method (Standard), a model with the gradient penalty regularization proposed by \cite{roth2017stabilizing} (GAN+GP) and a model with spectral normalization (SNGAN). To the best of our knowledge, these methods are the current state-of-the-art in terms of GAN stabilization. Fig.~\ref{fig:exp} shows that our method is stable and accurate across all settings.

\begin{figure*}[!h]
    \centering
    \includegraphics[width=\linewidth,trim=5.5cm 0 5.5cm 0,clip ]{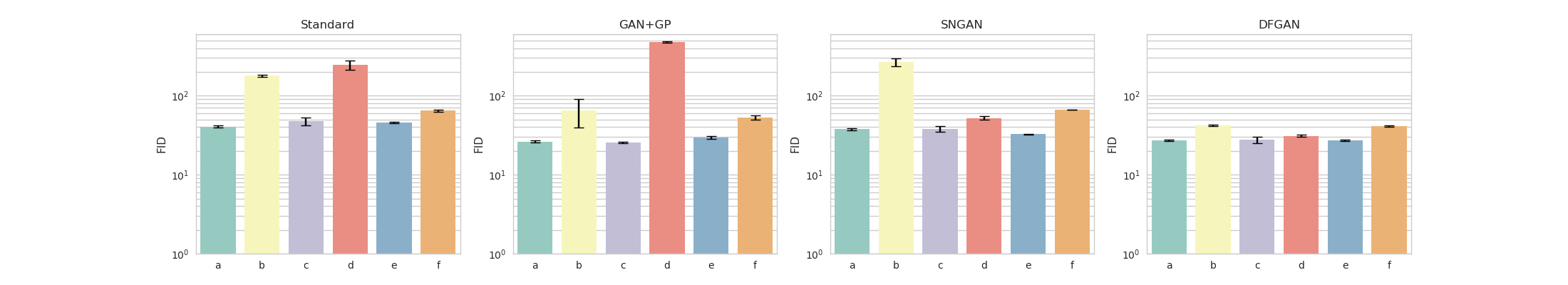}
    \includegraphics[width=\linewidth,trim=5.5cm 0 5.5cm 0,clip]{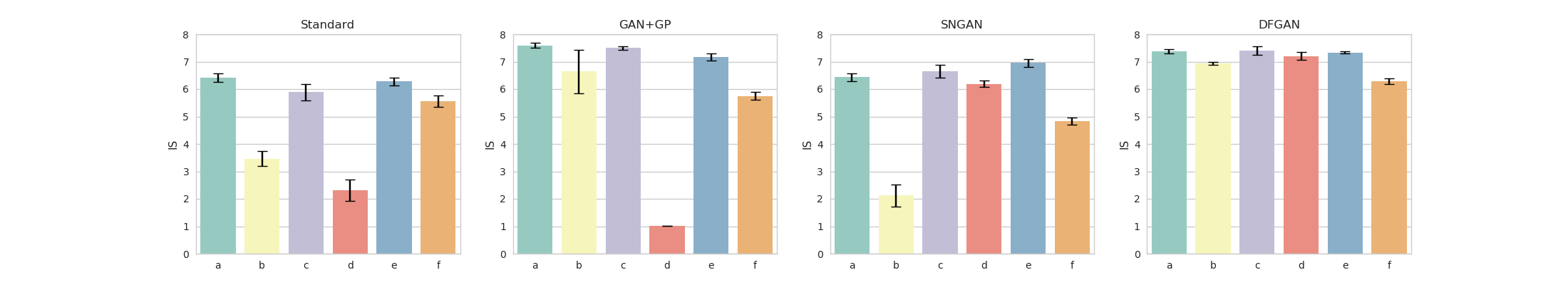}
    \caption{Results of the robustness experiments in Table \ref{tab:set} on CIFAR-10. We compare the standard GAN (\emph{1st column}), a GAN with gradient penalty (\emph{2nd column}), a GAN with spectral normalization (\emph{3rd column}) and a GAN with our proposed method (\emph{4th column}). Results are reported in Fr\'{e}chet Inception Distance \texttt{FID} (\emph{top}) and Inception Score \texttt{IS} (\emph{bottom}).}
  	\label{fig:exp}
\end{figure*}

\noindent\textbf{Robustness to Network Architectures.} To test the robustness of DFGANs against non-optimal network architectures we modified the networks in Table~\ref{tab:cifar_net} by doubling the number of layers in both generator and discriminator.  This leads to significantly worse performance in terms of FID in all cases: 46 to 135 (Standard), 33 to 111 (SNGAN), 28 to 36 (GAN+GP), and 27 to 60 (DFGAN). However, SNGAN+DF leads to good results with a FID of 27.6.


\begin{table}[t]
\centering
\caption{Hyperparameter settings used to evaluate the robustness of our proposed GAN training method. We vary the learning rate $\alpha$, the normalization in $G$, the optimizer, the activation functions, the number of discriminator iterations $n_{disc}$ and the number of training examples $n_{train}$. }
\label{tab:set}
\resizebox{\linewidth}{!}{\begin{tabular}{@{}lcccccc@{}}
\toprule
Exp. & \textbf{LR $\alpha$} & \textbf{BN in $G$} & \textbf{Opt.} & \textbf{ActFn}   &    \textbf{$n_{disc}$}   &    $n_{train}$                          \\ \midrule
a)  & $2\cdot 10^{-4}$  & FALSE    & ADAM     & (l)ReLU 	&	1	&	50K	\\
b)  & $2\cdot 10^{-4}$  & TRUE    & ADAM     & tanh		&	1	& 50K	\\
c)  & $1\cdot 10^{-3}$  & TRUE    & ADAM     & (l)ReLU 	&	1 	&50K	\\
d)  & $1\cdot 10^{-2}$  & TRUE    & SGD     & (l)ReLU 	&	1 	&50K	\\
e)  & $2\cdot 10^{-4}$  & TRUE    & ADAM     & (l)ReLU 	&	5 	&50K	\\
f)  & $2\cdot 10^{-4}$  & TRUE    & ADAM     & (l)ReLU 	&	1	&	5K	\\
\bottomrule
\end{tabular}}
\end{table}

\subsection{Qualitative Results} We trained DFGANs on $128\times128$ images from two large-scale datasets: ImageNet \cite{russakovsky2015imagenet} and LSUN bedrooms \cite{Yu2015LSUNCO}. The network architecture is similar to the one in Table~\ref{tab:cifar_net} with one additional layer in both networks. We trained the models for 100K iterations on LSUN and 300K iterations on ImageNet. Random samples of the models are shown in Fig.~\ref{fig:lsun_imnet}.
In Fig.~\ref{fig:noise} we show some examples of the noise that is produced by the noise generator at different stages during training. These examples resemble the image patterns that typically appear when the generator diverges.

\begin{figure}[t]
    \centering
    \begin{subfigure}[t]{\linewidth}
        \centering
		\adjincludegraphics[width=0.09\linewidth, trim={0 {.75\height} 0 0},clip]{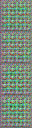}
		\adjincludegraphics[width=0.09\linewidth, trim={0 {.75\height} 0 0},clip]{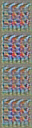}
		\adjincludegraphics[width=0.09\linewidth, trim={0 {.75\height} 0 0},clip]{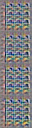}
		\adjincludegraphics[width=0.09\linewidth, trim={0 {.75\height} 0 0},clip]{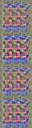}
		\adjincludegraphics[width=0.09\linewidth, trim={0 {.75\height} 0 0},clip]{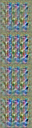}
		\adjincludegraphics[width=0.09\linewidth, trim={0 {.75\height} 0 0},clip]{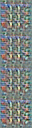}
		\adjincludegraphics[width=0.09\linewidth, trim={0 {.75\height} 0 0},clip]{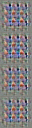}
		\adjincludegraphics[width=0.09\linewidth, trim={0 {.75\height} 0 0},clip]{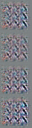}
		\adjincludegraphics[width=0.09\linewidth, trim={0 {.75\height} 0 0},clip]{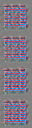}
		\adjincludegraphics[width=0.09\linewidth, trim={0 {.75\height} 0 0},clip]{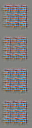}
    \end{subfigure}
    \begin{subfigure}[t]{\linewidth}
        \centering
		\adjincludegraphics[width=0.09\linewidth, trim={0 0 {.875\width} 0},clip]{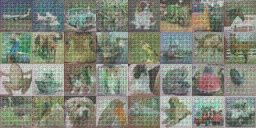}
		\adjincludegraphics[width=0.09\linewidth, trim={0 0 {.875\width} 0},clip]{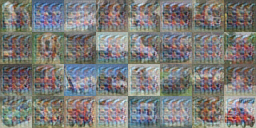}
		\adjincludegraphics[width=0.09\linewidth, trim={0 0 {.875\width} 0},clip]{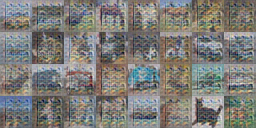}
		\adjincludegraphics[width=0.09\linewidth, trim={0 0 {.875\width} 0},clip]{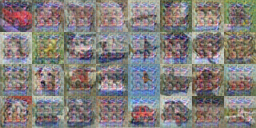}
		\adjincludegraphics[width=0.09\linewidth, trim={0 0 {.875\width} 0},clip]{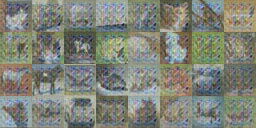}
		\adjincludegraphics[width=0.09\linewidth, trim={0 0 {.875\width} 0},clip]{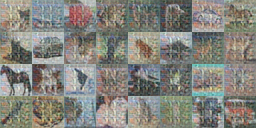}
		\adjincludegraphics[width=0.09\linewidth, trim={0 0 {.875\width} 0},clip]{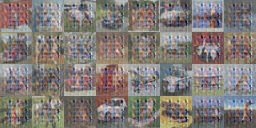}
		\adjincludegraphics[width=0.09\linewidth, trim={0 0 {.875\width} 0},clip]{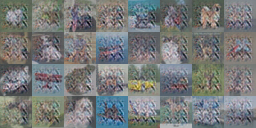}
		\adjincludegraphics[width=0.09\linewidth, trim={0 0 {.875\width} 0},clip]{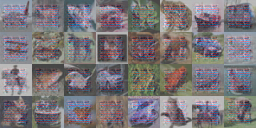}
		\adjincludegraphics[width=0.09\linewidth, trim={0 0 {.875\width} 0},clip]{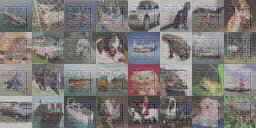}
    \end{subfigure}
    \caption{Examples of the generated noise (top row) and corresponding noisy training examples (rows 2 to 4). The columns correspond to different iterations. The noise varies over time to continually challenge the discriminator.}
  	\label{fig:noise}
\end{figure}

\section{Conclusions}

We have introduced a novel method to stabilize generative adversarial training that results in accurate generative models. Our method is rather general and can be applied to other GAN formulations with an average improvement in generated sample quality and variety, and training stability. Since GAN training aims at matching probability density distributions, we add random samples to both generated and real data to extend the support of the densities and thus facilitate their matching through gradient descent. 
We demonstrate the proposed training method on several common datasets of real images. \\

\noindent\textbf{Acknowledgements.} This work was supported by the Swiss National Science Foundation (SNSF) grant number 200021\_169622.
We also wish to thank Abdelhak Lemkhenter for discussions and for help with the proof of Theorem~1.

{\small
\bibliographystyle{ieee_fullname}
\bibliography{refs}
}

%

\newpage

\begin{center}
  \textbf{\large Supplementary Material for: \\ On Stabilizing Generative Adversarial Training with Noise}\\
\end{center}

\setcounter{equation}{0}
\setcounter{figure}{0}
\setcounter{table}{0}
\setcounter{page}{1}
\setcounter{section}{0}




\section{Influence on the Generator Gradient Norm }

We compare the norm of the generator gradient with and without DF for a GAN trained with the original minimax objective and a GAN trained with the alternative generator objective $\max_G \log(D(z))$  in Figure \ref{fig:gnorm}. The models were trained on CIFAR-10. We can observe the vanishing gradient phenomenon in Figure \ref{fig:gnorm1} when no distribution filtering is applied. With our proposed method the gradient norms are stable. In the case of the alternative loss in Figure \ref{fig:gnorm2} we can observe that the gradient norms are orders of magnitude higher when no distribution filtering is applied. This results in highly unstable weight updates due to the overconfident discriminator. 

\begin{figure}[h!]
    \centering
    \begin{subfigure}[t]{\linewidth}
        \centering
        \includegraphics[width=.9\linewidth ]{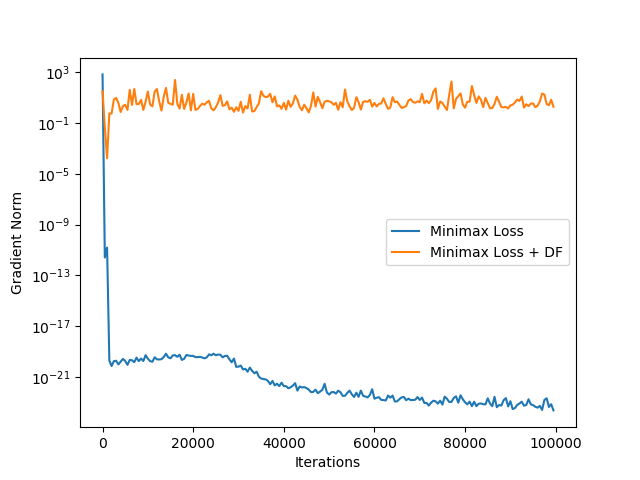}
        \caption{GAN with minimax loss}\label{fig:gnorm1}
    \end{subfigure}
    \begin{subfigure}[t]{\linewidth}
        \centering
        \includegraphics[width=.9\linewidth ]{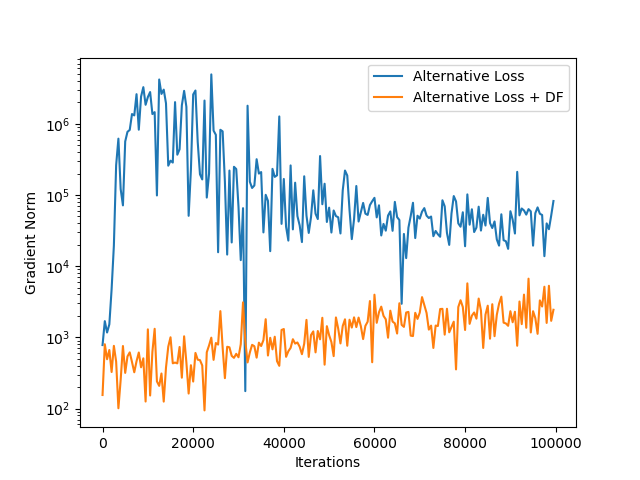}
        \caption{GAN with alternative loss}\label{fig:gnorm2}
    \end{subfigure}
    \caption{We show the norm of the generator gradient over the course of training for a GAN using the original minimax objective in (a) and a GAN using the alternative objective $\max_G \log(D(z))$ in (b). }
  	\label{fig:gnorm}
\end{figure}

\section{Experiments on synthetic data}

We performed experiments with a standard GAN and a DFGAN using Gaussian noise on synthetic 2-D data. The generator and discriminator architectures are both MLPs consisting of three fully-connected layers with a hidden-layer size of 512. We use ReLU activations and batch-normalization (\cite{ioffe2015batch}) in all but the first discriminator layer and the output layers. The Adam optimzer (\cite{kingma2014adam}) was used with a learning rate of $10^{-4}$ and we trained for 20K iterations. The results are shown in Figure \ref{fig:synth}. We can observe how the matching of both clean and filtered distribution leads to a better fit in the case of DFGAN.

\section{Implementation Details}

\noindent\textbf{Noise Generator.} The noise-generator architecture in all our experiments is equivalent to the generator architecture with the number of filters reduced by a factor of eight. The output of the noise-generator has a \texttt{tanh} activation scaled by a factor of two to allow more noise if necessary. We also experimented with a linear activation but didn't find a significant difference in performance.  

\noindent\textbf{GAN+GP.} For the comparisons to the GAN regularizer proposed by \cite{roth2017stabilizing} we used the same settings as used in their work in experiments with DCGAN. 

\noindent\textbf{SNGAN+DF.} We used the standard GAN loss (same as DCGAN) in all our experiments with models using spectral normalization. When combining SNGAN with DF we batch-normalized the noisy inputs to the discriminator.  

\section{Qualitative Results for Experiments}

We provide qualitative results for some of the ablation experiments in Figure \ref{fig:abl} and for the robustness experiments in Figure \ref{fig:exps}. As we can see in Figure \ref{fig:exps}, none of the tested settings led to degenerate solutions in the case of DFGAN while the other methods would show failure cases in some settings.  

\section{Application to Progressive GAN}

To test our method on a state-of-the-art GAN we applied our training method to the progressive GAN model. We used the DCGAN loss, trained for a total of 6M images and did not use label conditioning. We used fixed Gaussian noise for the distribution filtering. On CIFAR-10 progressiveGAN without DF achieved a FID of 29.4. Adding DF improved the performance to 26.8. Note that the original WGAN-GP loss in the same setup only achieved a FID of 29.8.

We also trained progressive-GAN+DF on higher resolution $256\times256$ images of LSUN bedrooms. See Figure \ref{fig:lsun} for results.

\begin{figure*}[h!]
    \centering
    \begin{subfigure}[t]{\textwidth}
        \centering
        \includegraphics[width=.48\linewidth ]{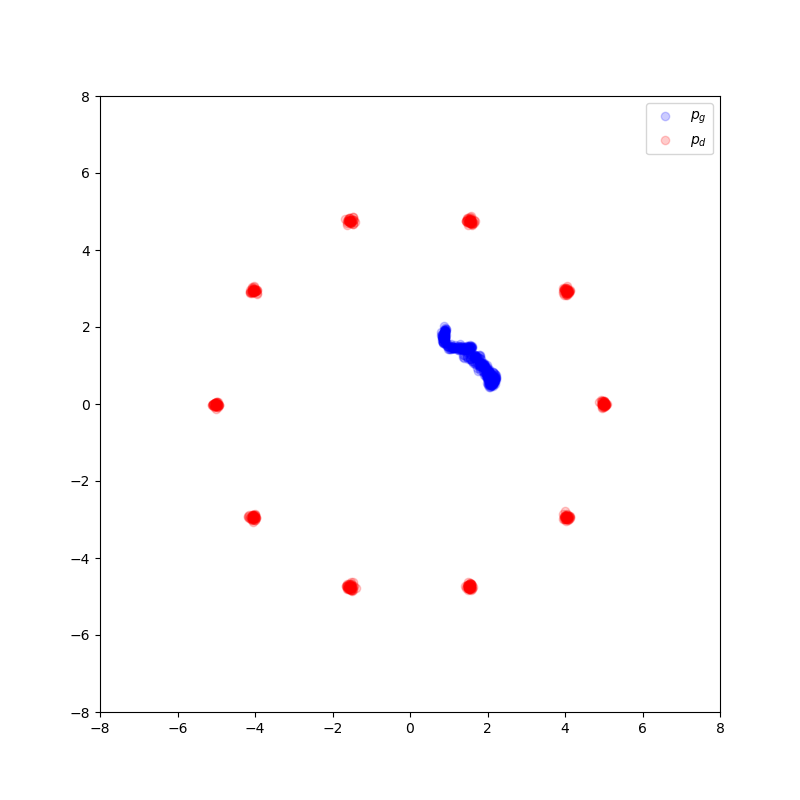}
        \includegraphics[width=.48\linewidth]{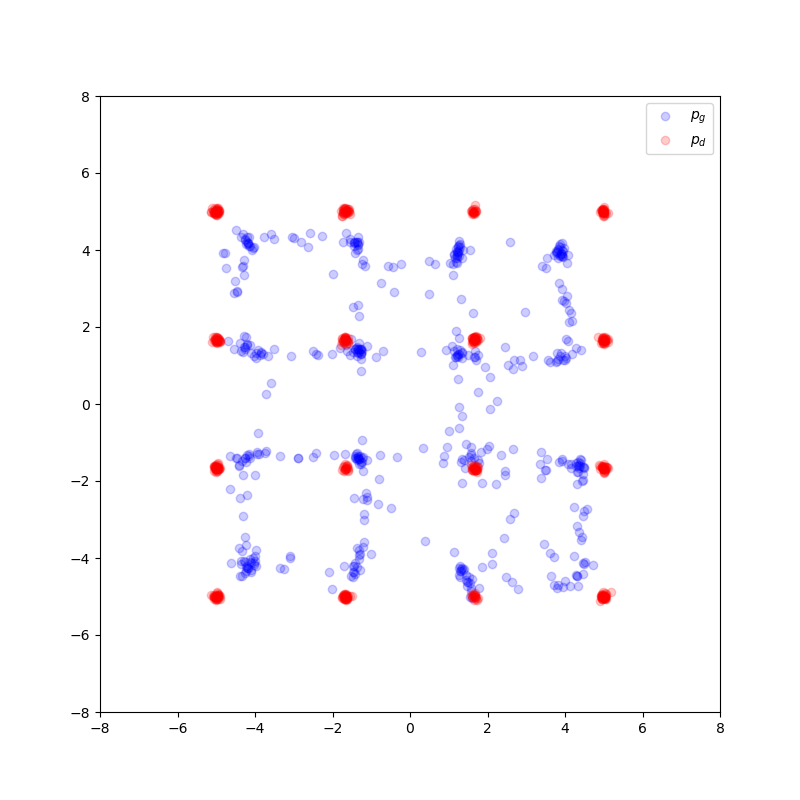}
        \caption{Standard GAN}
    \end{subfigure}
    \begin{subfigure}[t]{\textwidth}
        \centering
        \includegraphics[width=.48\linewidth ]{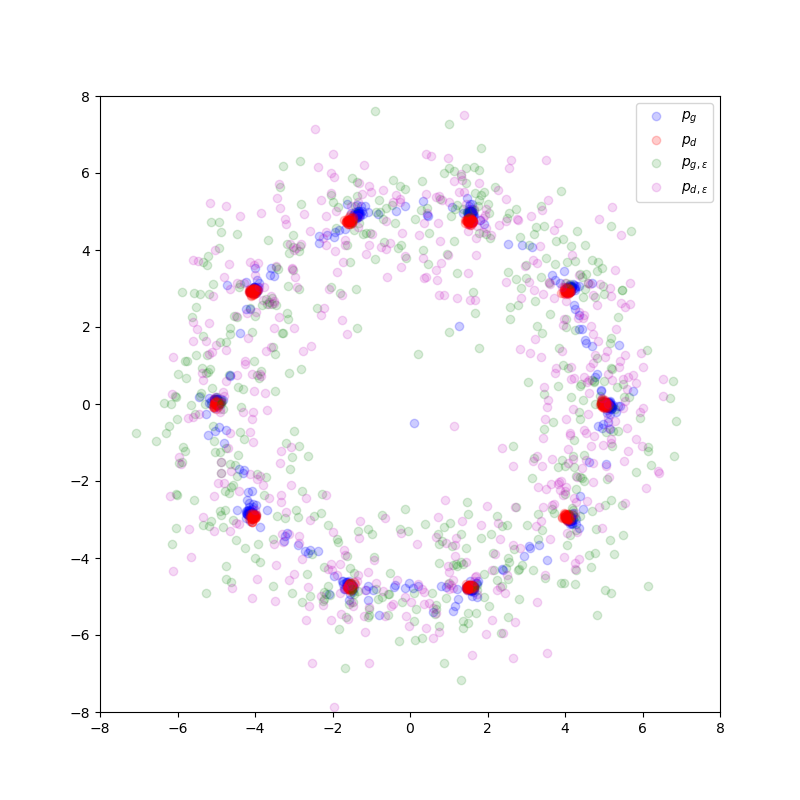}
        \includegraphics[width=.48\linewidth]{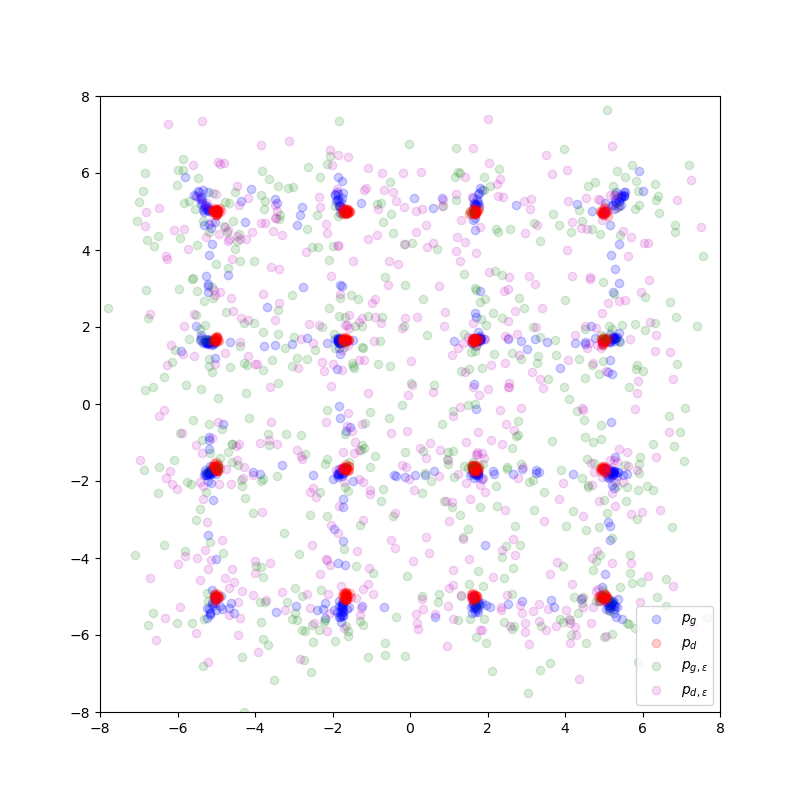}
        \caption{DFGAN with Gaussian noise}
    \end{subfigure}
    \caption{We performed experiments on synthetic 2D data with a standard GAN (\emph{top}) and a DFGAN (\emph{bottom}). The ground truth data is shown in \emph{red} and the model generated data is shown in \emph{blue}. For DFGAN we also show samples from the blurred data distribution $p_{d,\epsilon}$ in \emph{green} and the blurred model distribution $p_{g,\epsilon}$ in \emph{purple}. }
  	\label{fig:synth}
\end{figure*}

\begin{figure*}[h!]
    \centering
    \begin{subfigure}[t]{\textwidth}
        \centering
        \adjincludegraphics[width=.48\linewidth,trim={0 0 {.375\width} 0},clip]{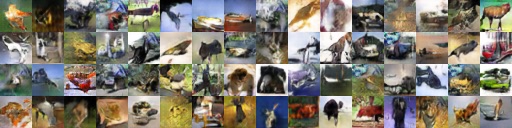}
        \adjincludegraphics[width=.48\linewidth,trim={0 0 {.375\width} 0},clip]{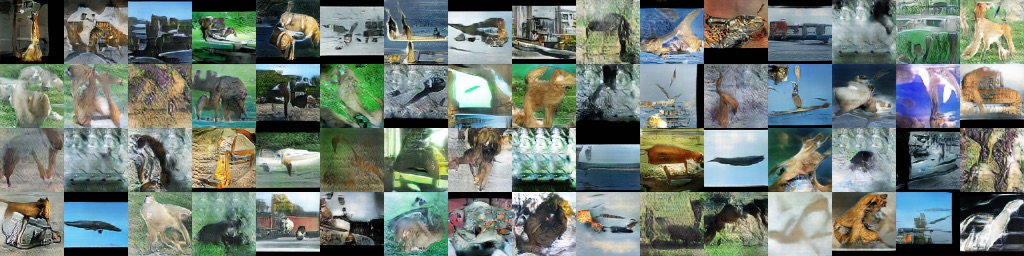}
        \caption{Standard GAN}
    \end{subfigure}
    \begin{subfigure}[t]{\textwidth}
        \centering
        \adjincludegraphics[width=.48\linewidth,trim={0 0 {.375\width} 0},clip]{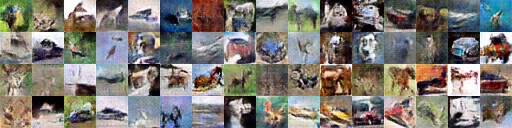}
        \adjincludegraphics[width=.48\linewidth,trim={0 0 {.375\width} 0},clip]{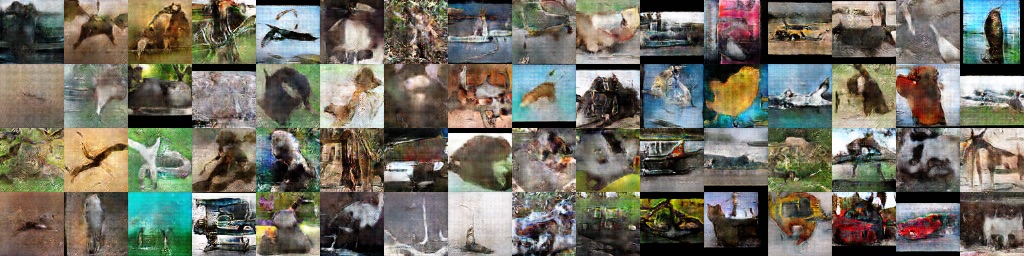}
        \caption{Noise only: $\epsilon \sim \mathcal{N}(0, I)$}
    \end{subfigure}
    \begin{subfigure}[t]{\textwidth}
        \centering
        \adjincludegraphics[width=.48\linewidth,trim={0 0 {.375\width} 0},clip]{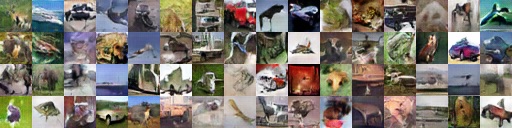}
        \adjincludegraphics[width=.48\linewidth,trim={0 0 {.375\width} 0},clip]{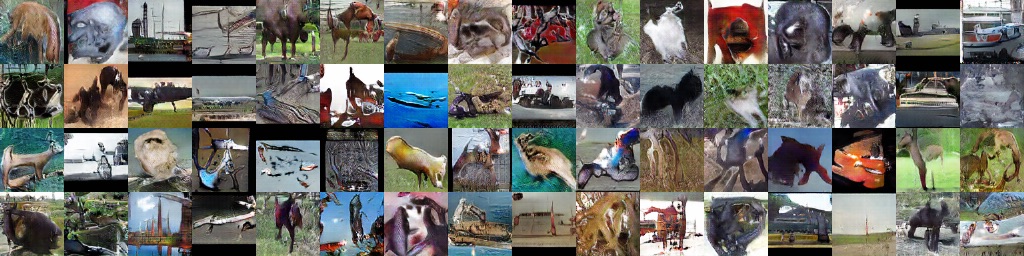}
        \caption{Noise only: $\epsilon \sim \mathcal{N}(0,\sigma I)$, $\sigma \rightarrow 0$}
    \end{subfigure}
    \begin{subfigure}[t]{\textwidth}
        \centering
        \adjincludegraphics[width=.48\linewidth,trim={0 0 {.375\width} 0},clip]{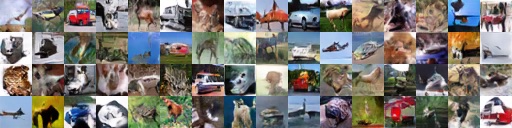}
        \adjincludegraphics[width=.48\linewidth,trim={0 0 {.375\width} 0},clip]{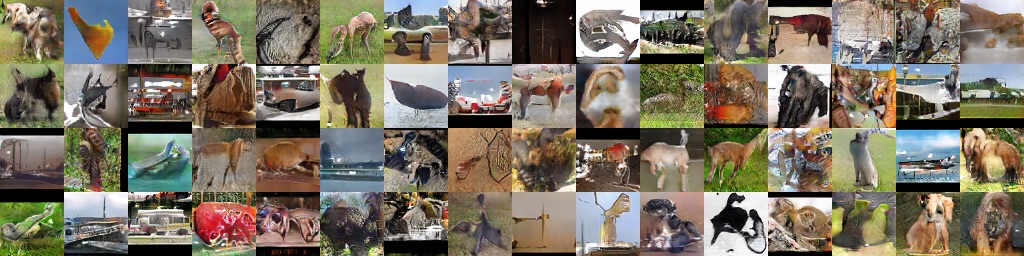}
        \caption{Clean + noise: $\epsilon \sim \mathcal{N}(0,I)$ (CIFAR-10)}
    \end{subfigure}
    \begin{subfigure}[t]{\textwidth}
        \centering
        \adjincludegraphics[width=.48\linewidth,trim={0 0 {.375\width} 0},clip]{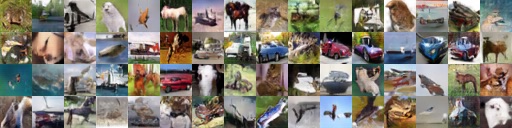}
        \adjincludegraphics[width=.48\linewidth,trim={0 0 {.375\width} 0},clip]{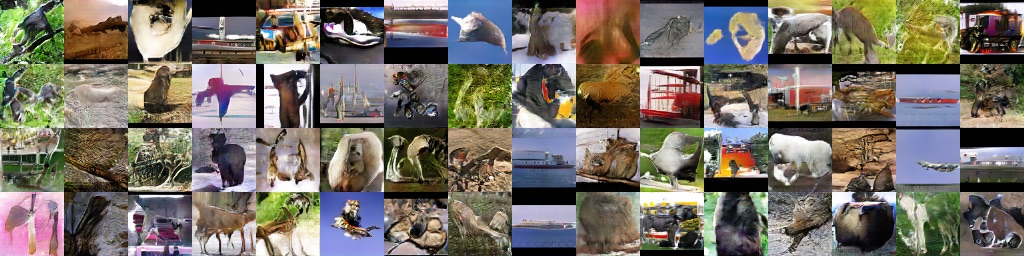}
        \caption{DFGAN $(\lambda=1 )$}
    \end{subfigure}
    \caption{We show random reconstructions for some of the ablation experiments listed in Table 2 of the paper. The left column shows results on CIFAR-10 and the right column shows results on STL-10.}
  	\label{fig:abl}
\end{figure*}

\begin{figure*}[h!]
    \centering
    \begin{subfigure}[t]{\textwidth}
        \centering
		\adjincludegraphics[width=4cm,trim={0 0 {.625\width} 0},clip]{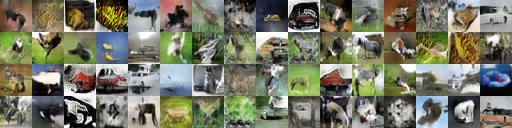}
		\adjincludegraphics[width=4cm,trim={0 0 {.625\width} 0},clip]{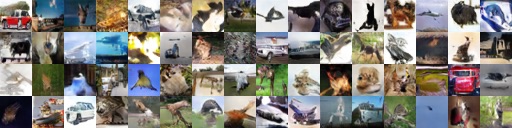}
		\adjincludegraphics[width=4cm,trim={0 0 {.625\width} 0},clip]{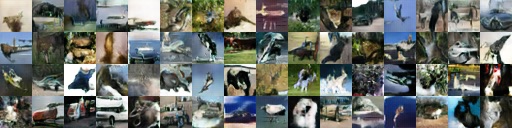}
		\adjincludegraphics[width=4cm,trim={0 0 {.625\width} 0},clip]{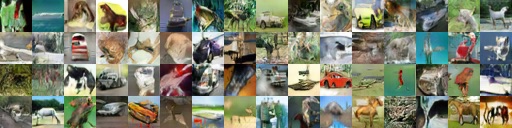}
	    \caption{}
    \end{subfigure}
        \begin{subfigure}[t]{\textwidth}
        \centering
		\adjincludegraphics[width=4cm,trim={0 0 {.625\width} 0},clip]{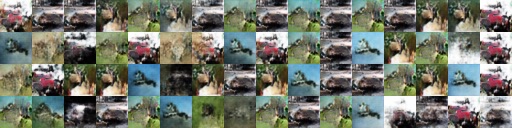}
		\adjincludegraphics[width=4cm,trim={0 0 {.625\width} 0},clip]{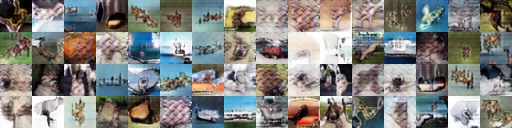}
		\adjincludegraphics[width=4cm,trim={0 0 {.625\width} 0},clip]{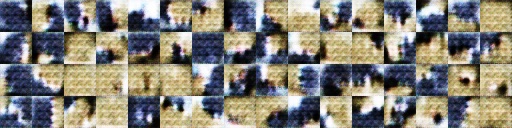}
		\adjincludegraphics[width=4cm,trim={0 0 {.625\width} 0},clip]{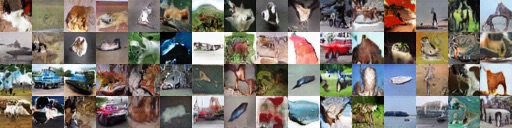}
	    \caption{}
    \end{subfigure}
        \begin{subfigure}[t]{\textwidth}
        \centering
		\adjincludegraphics[width=4cm,trim={0 0 {.625\width} 0},clip]{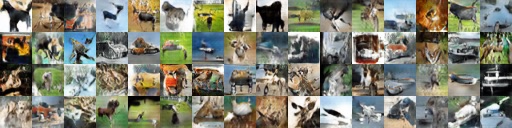}
		\adjincludegraphics[width=4cm,trim={0 0 {.625\width} 0},clip]{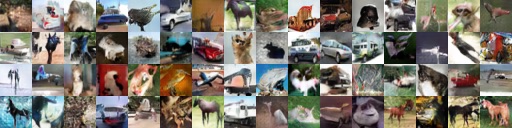}
		\adjincludegraphics[width=4cm,trim={0 0 {.625\width} 0},clip]{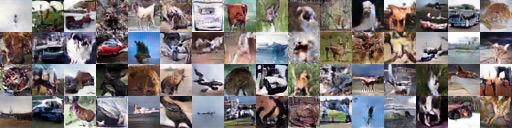}
		\adjincludegraphics[width=4cm,trim={0 0 {.625\width} 0},clip]{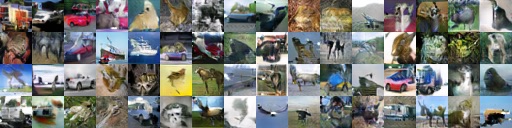}
	    \caption{}
    \end{subfigure}
        \begin{subfigure}[t]{\textwidth}
        \centering
		\adjincludegraphics[width=4cm,trim={0 0 {.625\width} 0},clip]{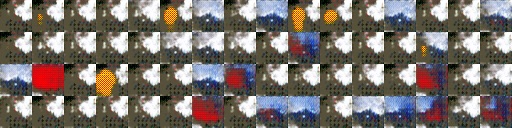}
		\adjincludegraphics[width=4cm,trim={0 0 {.625\width} 0},clip]{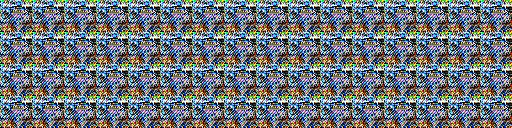}
		\adjincludegraphics[width=4cm,trim={0 0 {.625\width} 0},clip]{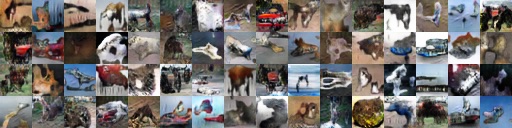}
		\adjincludegraphics[width=4cm,trim={0 0 {.625\width} 0},clip]{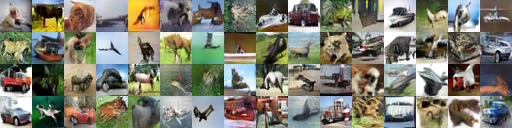}
	    \caption{}
    \end{subfigure}
        \begin{subfigure}[t]{\textwidth}
        \centering
		\adjincludegraphics[width=4cm,trim={0 0 {.625\width} 0},clip]{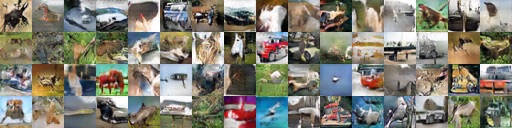}
		\adjincludegraphics[width=4cm,trim={0 0 {.625\width} 0},clip]{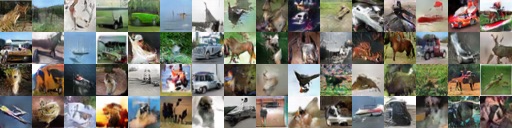}
		\adjincludegraphics[width=4cm,trim={0 0 {.625\width} 0},clip]{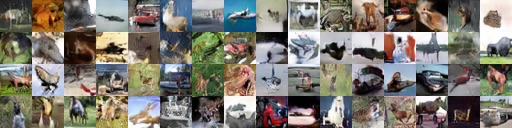}
		\adjincludegraphics[width=4cm,trim={0 0 {.625\width} 0},clip]{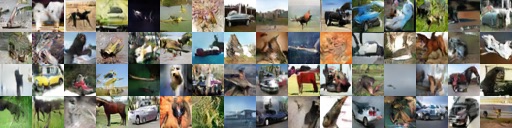}
	    \caption{}
    \end{subfigure}
        \begin{subfigure}[t]{\textwidth}
        \centering
		\adjincludegraphics[width=4cm,trim={0 0 {.625\width} 0},clip]{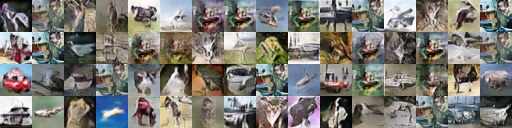}
		\adjincludegraphics[width=4cm,trim={0 0 {.625\width} 0},clip]{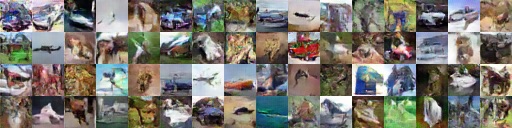}
		\adjincludegraphics[width=4cm,trim={0 0 {.625\width} 0},clip]{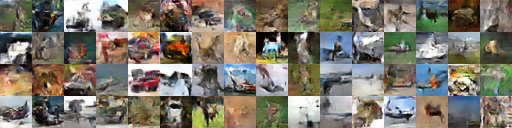}
		\adjincludegraphics[width=4cm,trim={0 0 {.625\width} 0},clip]{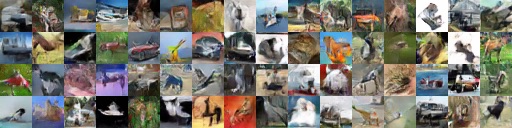}
	    \caption{}
    \end{subfigure}
    \caption{We show random reconstructions for the robustness experiments (see Table 4). We compare a standard GAN (\emph{1st column}), a GAN with gradient penalty by \cite{roth2017stabilizing} (\emph{2nd column}), a GAN with spectral normalization by \cite{miyato2018spectral} (\emph{3rd column}) and a GAN with our proposed method (\emph{4th column}). }
  	\label{fig:exps}
\end{figure*}

\begin{figure*}[h!]
    \centering
    \includegraphics[width=.98\linewidth ]{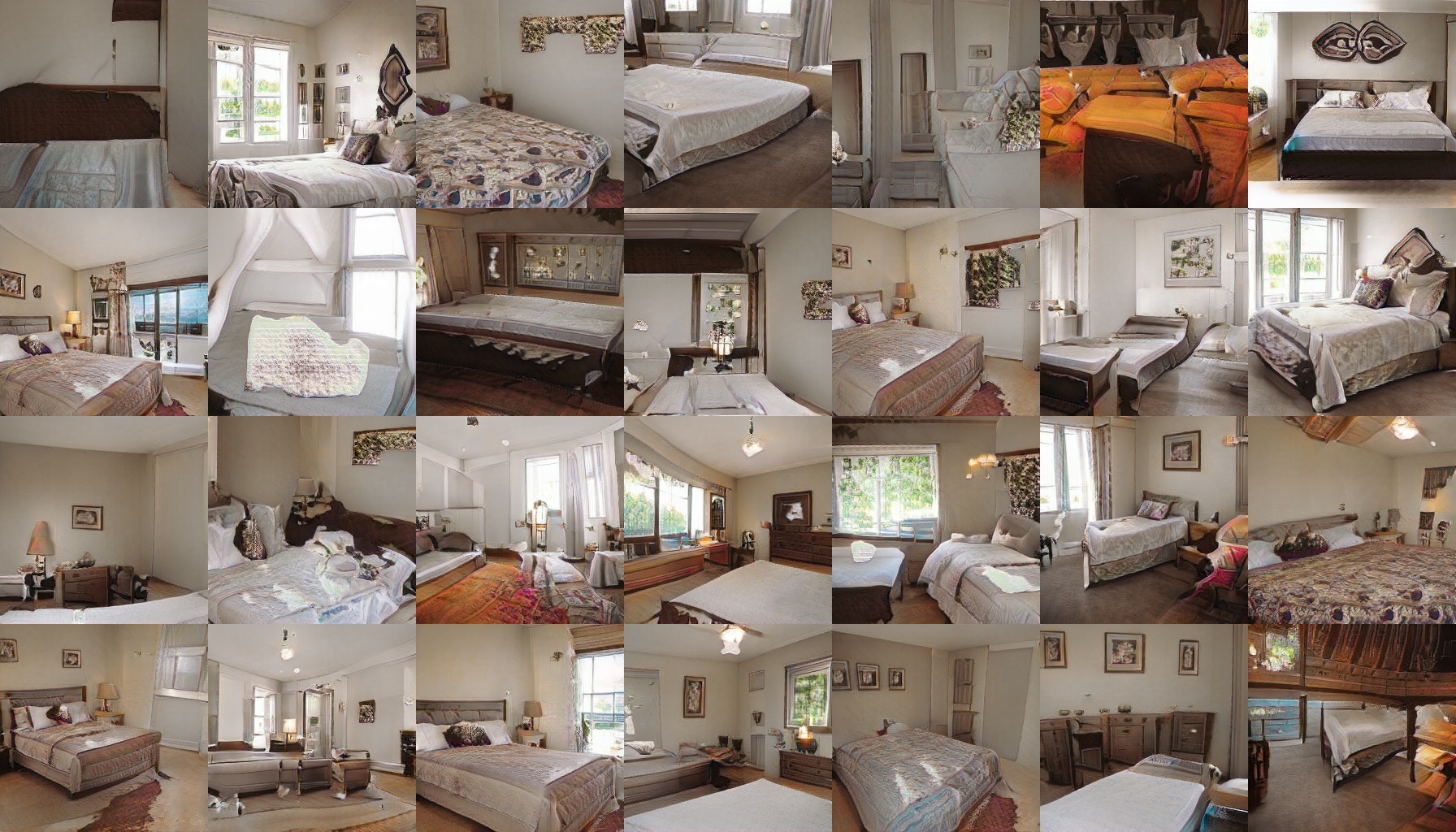}
    \caption{Results of progressive-GAN+DF trained on LSUN bedrooms. We used the DCGAN loss and Gaussian noise for DF in this experiment.}
  	\label{fig:lsun}
\end{figure*}

\end{document}